\documentclass{article}

\usepackage[ruled,vlined,linesnumbered]{algorithm2e}

\usepackage[margin=1in]{geometry}

\usepackage{wrapfig}
\usepackage{lipsum}
\usepackage{hhline}
\usepackage[utf8]{inputenc} 
\usepackage[T1]{fontenc}    
\usepackage{hyperref}       
\usepackage{url}            
\usepackage{booktabs}       
\usepackage{amsfonts}       
\usepackage{nicefrac}       
\usepackage{microtype}      
\usepackage{xcolor}         
\usepackage{titletoc}
\usepackage{cite}
\usepackage[font=small]{caption}

\usepackage[table]{xcolor}
\newcommand{\colorcolsA}[1]{\cellcolor[rgb]{0.85,0.9,0.9}#1}
\usepackage{multirow}

\SetKwInput{KwInput}{Input}
\SetKwInput{KwOutput}{Output}
\SetKwFor{GlobalFor}{each global round}{do}{}
\SetKwFor{LocalFor}{ each local iteration}{do}{}
\SetKwFor{ClientFor}{ each client}{in parallel do}{}

\title{Routing Without Training: Controllable-Ratio LLM Offloading via Reliability Gating}

\author{%
\begin{tabular}{c}
Evan Chen$^{1}$, Shiqiang Wang$^{2}$, Kevin S. Chan$^{3}$, Su Wang$^{4}$, Christopher G. Brinton$^{1}$\\[0.5em]
$^{1}$Purdue University\\
$^{2}$University of Exeter\\
$^{3}$Army Research Laboratory\\
$^{4}$Princeton University
\end{tabular}%
}

\usepackage{colortbl}
\usepackage{microtype}
\usepackage{graphicx}
\usepackage{subfigure}

\usepackage{booktabs} 
\usepackage{xcolor}
\usepackage{multirow}
\usepackage{makecell}

\usepackage{hyperref}
\usepackage{enumitem} 
\usepackage{subcaption} 

\usepackage{amsmath}
\usepackage{amssymb}
\usepackage{mathtools}
\usepackage{amsthm}
\usepackage[capitalize,noabbrev]{cleveref}

\theoremstyle{plain}
\newtheorem{theorem}{Theorem}[section]

\newtheorem{lemma}[theorem]{Lemma}

\theoremstyle{definition}

\usepackage{pifont}

\allowdisplaybreaks

\definecolor{lightred}{rgb}{1, 0.7, 0.4}
\definecolor{lightblue}{rgb}{0.7,0.7,1}

\definecolor{lightgreen}{rgb}{0.7, 1, 0.7}

\usepackage[textsize=tiny]{todonotes}

\begin{document}

\maketitle

\begin{abstract}
Local-cloud collaboration is a practical way to deploy large language models under resource constraints, but existing methods often rely on trained routers or collaboration-aware finetuning that tie routing behavior to a particular operating regime. In this work, we show that such training may be unnecessary: the local model's own inference-time agreement across sampled responses already provides a strong signal for deciding when to trust local execution vs when to offload to a stronger cloud model. We propose CARGO, a training-free routing framework that estimates this agreement through prompt-varied sampling, applies Bayesian early stopping for sample-efficient uncertainty control, and supports arbitrary target collaboration ratios through lightweight deployment-time calibration. Across diverse reasoning and question-answering tasks, multiple local LLM families and scales, and both pretrained and finetuned local models, CARGO consistently outperforms other training-free baselines and in several settings surpasses supervised learned routers. These results suggest that effective and adaptable local-cloud collaboration can emerge directly from the local model's intrinsic response behavior, without requiring an additional trained router.
\end{abstract}

\section{Introduction}
\label{sec:intro}

As large language model (LLM) services are increasingly deployed on edge devices and other resource-constrained platforms, practical systems must balance the efficiency of local inference with the capability of stronger cloud models. In practical deployments, relying exclusively on cloud inference is undesirable because it introduces communication overhead, higher latency, recurring monetary cost, and potential privacy or connectivity concerns. These issues are especially pronounced in settings such as mobile devices, on-premise assistants, and edge applications, where local responsiveness and reduced dependence on remote infrastructure are often essential~\cite{zhang2024llm,zhu2024llava,zhang2024tinyllama,liu2024mobilellm,xu2024device}. 
Local-cloud collaboration is a natural alternative that preserves the quality benefits of large remote models while reducing the cost of always invoking the cloud. For example, a lightweight local LLM can handle many routine user queries, while more difficult demands are selectively offloaded to a stronger cloud model~\cite{xu2024edgellm}. 
As such, the central question for effective local-cloud collaboration is: when should the system trust the local model, and when should it invoke cloud assistance?
Existing approaches answer this question through explicit routers or collaboration-aware training~\cite{fangbridging,ding2024hybrid}, but typically optimize collaborative behavior around a particular operating regime, creating a mismatch with real deployments where budget, latency tolerance, backend load, and service-level objectives are time-varying.

This limitation suggests that offloading should not be determined solely by a task-specific routing policy learned for one fixed regime. Instead, effective local-cloud collaboration requires a routing principle that remains valid as deployment requirements and target collaboration ratios change. In particular, the system should trust local inference when the local model can answer reliably and invoke cloud assistance otherwise, while still allowing the overall degree of collaboration to be adjusted as needed. 
Therefore, we ground local-cloud routing in an inference-time estimate of the local model's own reliability. This estimate reflects query-level capability and can be calibrated at deployment time to meet different target collaboration ratios, such as offloading $10\%$, $30\%$, or $50\%$ of queries. The question is whether this reliability-driven routing can be achieved without training a separate router at all. These observations motivate the following research questions.

\textbf{How can a system tell, at inference time, whether local models can independently and reliably answer a query?}
A central challenge is that the routing decision must be made without access to ground-truth correctness, using only the local model's own behavior at inference time. Moreover, the desired signal should reflect the model's intrinsic problem understanding rather than rely on task-specific supervision, retraining, or a separately learned router. The difficulty is therefore to identify an intrinsic and broadly applicable reliability signal that remains meaningful across tasks, model families, and deployment settings.

\textbf{How can such a reliability signal be used to make routing both efficient and adaptable at deployment time?}
Reliability estimation itself consumes computation: collecting more evidence can improve confidence in the estimate, at the cost of higher latency and local inference overhead. At the same time, practical systems require routing behavior that can adapt to changing collaboration budgets, rather than remaining tied to a single fixed operating point. The difficulty is therefore to use reliability information in a sample-efficient manner while still enabling flexible deployment-time control over when cloud assistance is invoked.

Our key insight is that effective local-cloud routing may not require training a separate router at all. By properly leveraging the local model's own response behavior, one can derive a training-free routing principle that remains effective across diverse tasks and local LLM backbones while still supporting flexible deployment-time collaboration. This suggests that the local model's intrinsic response agreement is itself a strong foundation for local-cloud collaboration, and in some cases may even provide a stronger routing signal than alternatives that are explicitly trained for specific task properties/distributions.

\textbf{Our Contributions.}
Building on this insight, we propose Collaboration-Adaptive Routing via Agreement-Guided Offloading (CARGO), a training-free local-cloud collaboration framework for controllable LLM offloading. CARGO turns agreement among prompt-varied local responses into a calibrated routing policy. It (i) estimates local reliability from response stability, (ii) reduces unnecessary sampling through Bayesian early stopping, and (iii) adjusts routing behavior through a lightweight deployment-time calibration step. Since CARGO operates only through prompts and generated outputs, it applies when local fine-tuning or post-training is unavailable and can even be added as an inference-time wrapper to existing LLM or agent-based applications. This design makes agreement-guided routing practical for resource-constrained settings, where latency, compute, and cloud-usage budgets may vary across devices, tasks, and deployment conditions. Our main contributions are as follows:

\begin{itemize}[leftmargin=*, itemsep=0pt, topsep=1pt]
    \item We identify prompt-varied response agreement as a transferable inference-time reliability signal for local-cloud routing, showing that it yields greater correctness and consistency than than self-reported confidence or Chain-of-Thought (CoT) steps (Sec.~\ref{sec:motivation}).

    \item We introduce a practical routing procedure that combines Bayesian early stopping with lightweight calibration, reducing the number of local generations needed for reliability estimation while supporting user-specified collaboration ratios without retraining (Sec.~\ref{sec:method}).

    \item We demonstrate that CARGO achieves strong accuracy under constrained cloud usage across diverse tasks and local LLM backbones. It consistently improves over training-free baselines and, in several settings, outperforms supervised learned routers while remaining adaptable at deployment time (Sec.~\ref{sec:exp}).
\end{itemize}

\textbf{Related Works.}
Existing work on local-cloud collaboration for practical, resource-constrained LLM deployments has pursued two main directions. The first uses explicit routing mechanisms, such as external classifiers, preference-trained routers, or confidence- and cost-aware query selection policies, to decide when inputs should be offloaded~\cite{ding2024hybrid,ong2024routellm,chen2023frugalgpt,oh2025repic}. The second integrates routing behavior into local-model training itself, training the local model to jointly improve task performance and offloading decisions during post-training~\cite{chen2026joint,fangbridging}.  
While effective, these approaches still rely on learned routing behavior, either through separately trained routing modules or through collaboration-aware post-training. Such learned routing can be costly to obtain and rigid at deployment time, especially when budgets, latency constraints, or target collaboration ratios change. 
In contrast, our focus is on whether routing can instead be grounded directly in an intrinsic inference-time reliability signal extracted from the local model's own response behavior, without training a router or fine-tuning the local model for collaboration.

Recent work on self-consistency shows that agreement across multiple sampled responses is a strong inference-time signal that can be extracted directly from a model's own generation behavior, without requiring additional supervision or fine-tuning~\cite{chen2023universal,wang2022self,taubenfeld2025confidence}. Beyond the original majority-vote formulation, subsequent studies show that this signal remains useful across more general generation settings and can be strengthened through improved aggregation or confidence-aware selection~\cite{toh2024not}. However, these works primarily study a local-only inference setting, where cross-sample agreement is used to select or refine the same model's final answer. Local-cloud collaboration introduces a new decision dimension: the system need not rely on the best local aggregate, but can instead route uncertain cases to a stronger cloud model. This shifts the role of agreement from answer aggregation to reliability-aware routing, where high agreement supports local execution and low agreement motivates cloud assistance. Since this decision requires repeated sampling on a resource-limited local model, sample-efficient agreement estimation becomes essential.

\begin{figure}
    \centering
    \includegraphics[width=0.98\linewidth]{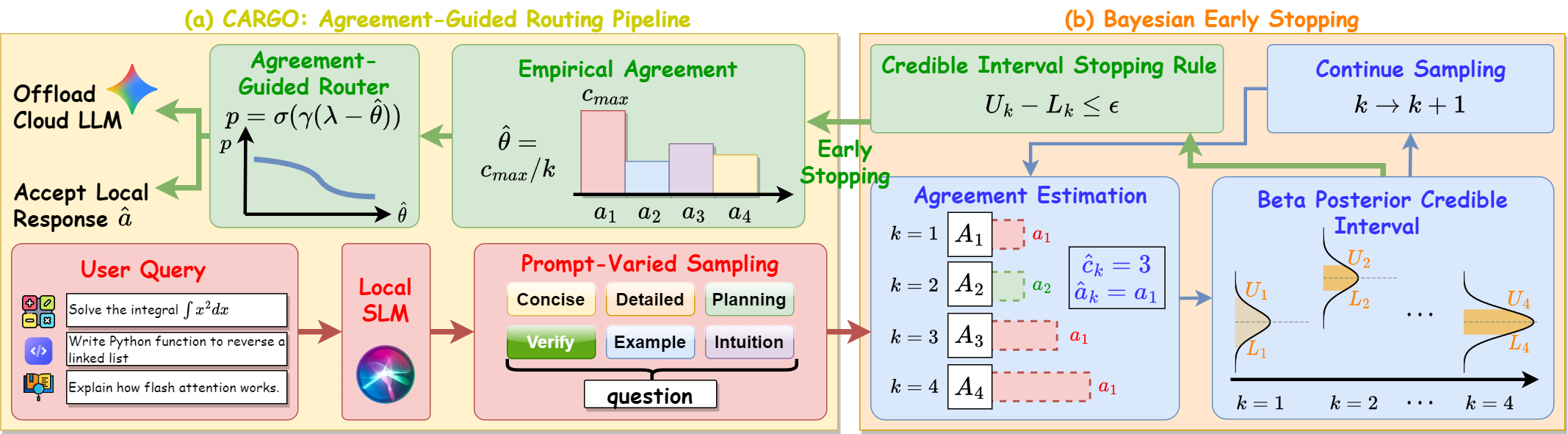}
    \caption{Overview of CARGO. The local model estimates agreement under prompt-varied sampling, then applies Bayesian early stopping to decide whether to accept the local response or offload to the cloud.}
    \label{fig:framework}
\end{figure}
\vspace{-0.1in}
\section{Motivation}
\vspace{-0.1in}

\label{sec:motivation}

\textbf{Training-Free Reliability Signals for Local-cloud Collaboration.}
Local-cloud collaboration has emerged as a practical paradigm for deploying LLMs under resource constraints. In such systems, a lightweight local model handles most queries while more difficult cases are routed to a stronger cloud model. A central challenge is therefore to design a routing mechanism that decides, for each query, whether the local response can be trusted.

\begin{wrapfigure}{r}{0.48\textwidth}
\vspace{-0.2in}
    \begin{center}
    \includegraphics[width=0.99\linewidth]{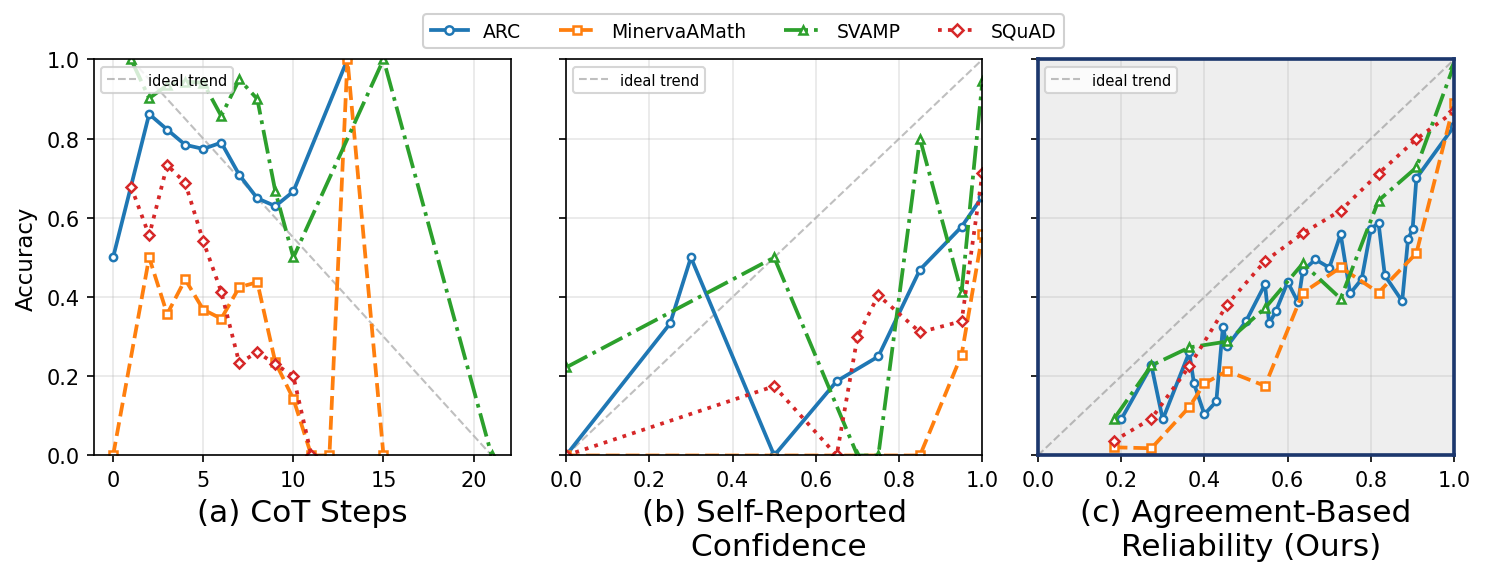}
    \end{center}
    \caption{Comparison of inference-time routing signals using Qwen2.5-7B-Instruct across multiple tasks. 
    Compared with CoT steps and self-reported confidence, response agreement exhibits a clearer monotonic relationship with accuracy, making it a stronger basis for reliability-guided routing.}
    \label{fig:motivation}
    \vspace{-0.2in}
\end{wrapfigure}

Many existing approaches rely on training an additional routing model or fine-tuning the local model to produce calibrated confidence scores. However, these approaches require task-specific supervision and retraining, limiting their applicability across domains and deployment conditions. We therefore ask whether the local model's own inference-time behavior already contains signals that can distinguish reliable from unreliable local predictions. To this end, we compare several candidate signals and evaluate how consistently they align with correctness across tasks.

\textbf{Existing Self-Reflection Signals.}
A natural hypothesis is that language models exhibit implicit indicators of task difficulty during generation. Two commonly used signals are Chain-of-Thought (CoT) Steps and Self-Reported Confidence. The first uses the number of intermediate reasoning steps under CoT prompting as a proxy for uncertainty, based on the heuristic that harder problems require longer reasoning~\cite{fangbridging}. The second explicitly asks the model to assess its own answer by outputting a confidence score~\cite{xiong2023can,zhao2024fact,tao2024trust}. However, as shown in Figure~\ref{fig:motivation}, both signals exhibit weak and unstable correlation with correctness in our setting, limiting their usefulness as training-free routing signals without additional task-specific adaptation or fine-tuning.

\textbf{Response Agreement for Reliability-Guided Routing.}
We leverage self-consistency, in which models encode their own reliability through response agreement across samples, to develop a principled methodology for local-cloud routing. Given a query $x$, a LLM generates an autoregressive answer by repeatedly selecting the next token from a probability distribution over the vocabulary; we refer to this generation procedure as decoding. Let $A \sim P(\cdot \mid x)$ denote the induced distribution over final answers. After sampling $k$ responses, we obtain counts $\{c_k^a\}_a$ over unique answers and define the empirical agreement level
\begin{equation}
\hat{\theta}_k = \max_a \frac{c_k^a}{k}.
\end{equation}
Here, $\hat{\theta}_k$ measures how strongly the sampled responses concentrate on a single answer. Empirically, $\hat{\theta}_k$ correlates much more strongly with correctness than CoT steps or self-reported confidence as shown in Figure~\ref{fig:motivation}, with high agreement typically corresponding to questions the local model can answer reliably and low agreement indicating ambiguity or difficulty.

\textbf{From Fixed Sampling to Adaptive Agreement Estimation.}
The usefulness of agreement as a reliability signal depends not only on how agreement is measured, but also on how efficiently it is estimated. Standard self-consistency typically draws a fixed number of samples $K$ for every query. However, queries differ in how quickly their agreement behavior becomes clear. For easy queries, sampled responses concentrate on the same answer after only a few generations, so additional samples provide little new routing information. For difficult or ambiguous queries, more evidence may be needed before the system can confidently decide whether the local model should be trusted.

This observation suggests that agreement-based routing should allocate sampling effort adaptively rather than uniformly. Such adaptivity is especially important in local-cloud collaboration, where repeated local generations introduce latency and compute overhead, while the system must still respect deployment-level constraints such as a target collaboration ratio. Instead of treating $K$ as a fixed budget for every input, CARGO sequentially samples only until the agreement estimate is sufficiently reliable, and then decides whether to accept the local prediction or invoke cloud assistance.

In the next section, we formalize these ideas by interpreting multi-sample agreement as an estimator of the mode (frequency) mass and developing a Bayesian framework to quantify its uncertainty. This leads to a Bayesian early-stopping procedure that terminates once the agreement estimate is sufficiently precise, after which routing decisions can be calibrated to a desired target collaboration ratio.

\begin{algorithm}[tb]
\caption{{\tt CARGO}: Collaboration-Adaptive Routing via Agreement-Guided Offloading}
\label{alg:cargo_alg}
\small
\DontPrintSemicolon
\SetKw{Break}{break}
\KwIn{
local model $M_L$; cloud model $M_C$; prompt set $\mathcal{S}=\{s_i\}_{i=1}^M$; temperature $T$; credible level $1-\delta$;
max samples $K_{\max}$; warmup iterations $T_0$; warmup batch size $B$; target offload ratio $\rho$; step size $\eta_\lambda$;
Beta prior $(\alpha_0,\beta_0)$; width threshold $\varepsilon$; logistic slope $\gamma$
}
\KwOut{for each query $x$: final answer $y(x)$ and routing indicator $r(x)\in\{0,1\}$ (1 = offload to cloud)}

\BlankLine
\textbf{Warmup (ratio control):} initialize intercept $\lambda\in\mathbb{R}$\;
sample a fixed warmup batch $\mathcal{B}=\{x^{(i)}\}_{i=1}^{B}$\;
\For{$t \leftarrow 1$ \KwTo $T_0$}{
    \ForEach{$x \in \mathcal{B}$}{
        $(y(x), r(x)) \leftarrow \textsc{RouteQuery}(x,\lambda)$\;
    }
    compute empirical $\bar r_t \leftarrow \frac{1}{B}\sum_{x\in\mathcal{B}} r(x)$, update threshold  $\lambda \leftarrow \lambda - \eta_\lambda(\bar r_t-\rho)$
}
\textbf{Deploy:} fix $\lambda$ for subsequent queries.\;

\BlankLine
\SetKwFunction{FRoute}{\textsc{RouteQuery}}
\SetKwProg{Fn}{Function}{:}{}
\Fn{\FRoute{$x,\lambda$}}{
    initialize multiset $\mathcal{A}\leftarrow \emptyset$ and counts $c(\cdot)\leftarrow 0$\;
    \For{$k \leftarrow 1$ \KwTo $K_{\max}$}{

        sample prompt template $s_k \sim \mathrm{Unif}(\mathcal{S})$, sample response $A_k \sim M_L(\cdot \mid x,s_k;T)$\;
        update counts: $\mathcal{A}_k\!\leftarrow\!\mathcal{A}_{k-1}\cup\{A_k\}$, $c(A_k)\!\leftarrow\!c(A_k)+1$\;
update counts and empirical agreement $(\hat a_k, \hat c_k) \leftarrow (\arg\max_a c_k^a,\; \max_a c_k^a)$\;
compute posterior $\theta|\mathcal{A}_k\sim\mathrm{Beta}(\alpha_0+\hat c_k,\beta_0+k-\hat c_k)$
and credible interval $[L_k,U_k]$\;

        \If{$U_k-L_k \le \varepsilon$}{
            \Break\tcp*{confidence-based early exit}
        }
    }

compute agreement estimate $\hat\theta \leftarrow \hat c_k/k$ and offload probability $p \leftarrow \sigma\!\big(\gamma(\lambda-\hat\theta)\big)$.

    sample $r \sim \mathrm{Bernoulli}(p)$\;
    \eIf{$r=0$}{
        \Return $(\hat a,0)$ \tcp*{accept local answer}
    }{
        query cloud $y \sim M_C(\cdot \mid x)$\;
        \Return $(y,1)$ \tcp*{offload to cloud}
    }
}
\end{algorithm}

\vspace{-0.1in}
\section{CARGO: Training-Free Agreement-Guided Routing}
\vspace{-0.1in}

\label{sec:method}

We consider a local-cloud inference setting with a lightweight local model $M_L$ and a stronger cloud model $M_C$. Given a query $x$, the system must decide whether to accept the local prediction or offload to the cloud. To this end, CARGO estimates reliability through multi-sample agreement, using (i) prompt variation, (ii) Bayesian mode-mass estimation, and (iii) Bayesian early stopping with calibrated probabilistic routing. The complete procedure is summarized in Algorithm~\ref{alg:cargo_alg}.

\textbf{Diverse Sampling via Prompt Variation.}
Sec.~\ref{sec:motivation} shows that agreement across multiple generations provides a useful reliability signal, but estimating this agreement requires informative variation across responses. A standard way to obtain such variation is temperature sampling, where higher temperature increases randomness in token selection. However, this randomness can perturb the model's native generation behavior and reduce the quality of individual responses. At the other extreme, deterministic decoding with a fixed prompt often produces identical outputs, making agreement uninformative. CARGO resolves this tension by combining deterministic decoding with prompt variation: we sample across semantically equivalent system prompts that preserve the task while encouraging different reasoning styles. Thus, we induce diversity across reasoning trajectories and final answers without relying on stochastic sampling noise.

Our approach offers two key advantages. First, deterministic decoding preserves the model's native generation behavior, ensuring that individual response quality is not degraded by temperature-induced perturbations. Second, by decoupling diversity from temperature, agreement estimates better reflect the model's inherent response stability rather than artifacts of sampling randomness.

Concretely, we define a small set of semantically equivalent system prompts $\mathcal{S}=\{s_1,\dots,s_M\}$, each encouraging a different reasoning style while preserving task semantics. For a query $x$, we sample a prompt $s_k \sim \mathrm{Unif}(\mathcal{S})$ and generate the response deterministically as $A_k=M_L(x,s_k;T=0)$. Across prompts and samples, we accumulate responses into a pooled set $\mathcal{A}_k=\{A_1,\dots,A_k\}$, where $k$ denotes the number of responses generated so far.

\textbf{Bayesian Estimation of Mode Mass.}
Given the pooled sample set $\mathcal{A}_k$, we compute counts $\{c_k^a\}_a$ over unique answers. Let $(\hat a_k,\hat c_k) = (\arg\max_a c_k^a,\; \max_a c_k^a)$ denote the empirical majority answer and its count. The empirical agreement estimator is then $\hat{\theta}_k=\hat c_k/k$. We interpret the \emph{mode mass} $\theta = P(A=\hat{a}_k \mid x)$ as the probability that the model outputs its most frequent answer under the same sampling procedure. Conditioned on $\hat{a}_k$, define the indicator $Z_i = \mathbf{1}\{A_i=\hat{a}_k\}$. Under conditional independence of stochastic generations, then $Z_i \sim \mathrm{Bernoulli}(\theta)$. We place a conjugate prior $\theta \sim \mathrm{Beta}(\alpha_0,\beta_0)$, where $\alpha_0$ and $\beta_0$ act as prior pseudo-counts for observing samples that match or do not match the empirical majority answer, respectively. Given $\hat c_k$ successes out of $k$ samples, conjugacy yields the posterior 
\begin{equation}
\textstyle \theta \mid \mathcal{A}_k
\sim
\mathrm{Beta}\!\left(\alpha_0+\hat c_k,\;
\beta_0+k-\hat c_k\right).
\end{equation}
The posterior mean provides a prior-smoothed estimate of the empirical agreement $\hat{\theta}_k$, while the posterior uncertainty decreases as more samples are collected.

\begin{lemma}[Posterior Contraction of Mode Mass]
\label{lem:posterior_contraction}
Assume decoding samples $\{A_i\}_{i=1}^k$ are conditionally independent. Under the Beta prior $\theta \sim \mathrm{Beta}(\alpha_0,\beta_0)$ with fixed hyperparameters, the posterior variance satisfies
\begin{equation}
\textstyle \mathrm{Var}(\theta \mid \mathcal{A}_k)\leq \frac{1}{4(k+\alpha_0+\beta_0+1)}.\notag
\end{equation}
\end{lemma}
The proof is provided in  Appendix~\ref{appen:post_pf}. This contraction result is useful for our routing procedure because CARGO stops sampling based on posterior uncertainty rather than a fixed sample count. Lemma~\ref{lem:posterior_contraction} shows that each additional response increases the effective posterior sample size, causing uncertainty around the agreement estimate to shrink at rate $O(1/k)$. The bound depends on the total prior concentration $\alpha_0+\beta_0$, which controls uncertainty, while the prior mean $\alpha_0/(\alpha_0+\beta_0)$ can still affect finite-sample estimates and early routing decisions. Thus, the lemma provides a principled basis for the credible-interval stopping rule used by CARGO: sampling continues only until the agreement estimate is sufficiently reliable.

\textbf{Bayesian Early Stopping for Agreement Estimation.}
While agreement provides a useful reliability signal, generating a large number of samples per query can increase latency. To reduce unnecessary sampling, we adopt a Bayesian early-stopping procedure that stops once the uncertainty of the mode-mass estimate becomes sufficiently small.
Given the posterior distribution $\theta \mid \mathcal{A}_k$, we compute a $(1-\delta)$ credible interval $[L_k,U_k]$. The algorithm continues sampling until the posterior uncertainty becomes small:
$U_k - L_k \le \varepsilon$.

The following result shows that this procedure yields a consistent estimate of the mode mass and terminates after a finite number of samples with probability one.

\begin{theorem}[Consistency of Bayesian Early-Stopped Mode-Mass Estimation]
\label{thm:sequential_estimation_consistency}
Let $\theta^\star=\max_a P(a\mid x)$ denote the true mode mass and assume decoding samples $\{A_i\}$ are conditionally independent. Let $[L_k,U_k]$ be the $(1-\delta)$ credible interval of the Beta posterior $\theta\mid\mathcal{A}_k$. Then $\hat{\theta}_k \to \theta^\star$ and $U_k-L_k\to0$ almost surely as $k\to\infty$. Consequently, the stopping rule $U_k-L_k\le\varepsilon$ terminates almost surely and returns a consistent estimate of $\theta^\star$.
\end{theorem}
The proof is provided in Appendix~\ref{appen:thm_pf}. Intuitively, this condition ensures that the agreement level has been estimated with sufficient confidence. As new samples are collected, the posterior uncertainty around the mode mass decreases, which narrows the credible interval and guarantees that the CARGO sampling process eventually terminates.
This stopping rule decouples sampling cost from the routing decision: sampling stops once agreement has been estimated reliably, regardless of whether the query is ultimately processed locally or offloaded to the cloud. After termination at sample $k$, the agreement estimator $\hat{\theta}_k$ provides a compact summary of the local model's response stability.

\textbf{Probabilistic Routing Policy.}
Instead of using a deterministic threshold on the agreement estimate, we employ a probabilistic routing policy that maps agreement to an offloading probability. Specifically, given agreement estimate $\hat{\theta}$, we define the offloading probability
$p = \sigma\!\big(\gamma(\lambda-\hat{\theta})\big),$
with sigmoid function $\sigma(z)=1/(1+e^{-z})$. The routing decision is then sampled as $r \sim \mathrm{Bernoulli}(p)$, where $r=1$ denotes cloud offloading and $r=0$ denotes accepting the local answer.
This formulation has several desirable properties. First, the agreement estimate directly controls routing behavior: high agreement implies $\hat{\theta}>\lambda$, resulting in a small offloading probability, while low agreement increases the likelihood of cloud routing. Second, the parameter $\lambda$ acts as an interpretable \emph{agreement pivot}: when $\hat{\theta}=\lambda$, the router is indifferent between local and cloud execution, since $p=0.5$. Third, compared with hard thresholding, this probabilistic formulation avoids instability when many queries have similar agreement levels. It also enables control of the global routing rate, which enables CARGO to meet varying system-level resource constraints.

\textbf{Warmup Calibration of the Routing Parameter.}
The intercept parameter $\lambda$ determines the overall routing behavior of the system. Increasing $\lambda$ raises the offloading probability for all queries, while decreasing $\lambda$ favors local execution. To achieve a target collaboration ratio $\rho$, we calibrate $\lambda$ using a short warmup phase. A fixed batch of $B$ warmup queries $\mathcal{B}$ is sampled, and the router is repeatedly applied to this batch while updating $\lambda$.
Let $r(x;\lambda)\in\{0,1\}$ denote the routing decision for query $x$. At iteration $t$, we compute the empirical offloading rate
$
\bar r_t=\frac{1}{B}\sum_{x\in\mathcal{B}} r(x;\lambda_t).
$
The routing parameter is then updated via stochastic approximation:
\begin{equation}
\lambda_{t+1} = \lambda_t - \eta_\lambda(\bar r_t-\rho).
\end{equation}
This update adjusts $\lambda$ in the direction that reduces the gap between the observed offloading rate and the desired target $\rho$. When the system offloads too frequently ($\bar r_t>\rho$), $\lambda$ is reduced, lowering the offloading probability. Conversely, if routing is too conservative, $\lambda$ increases. After $T_0$ warmup iterations, the calibrated $\lambda$ is fixed for deployment queries.
This calibration procedure allows the system to satisfy a user-specified target collaboration ratio without model retraining or task-specific tuning. Because the routing probability is smooth in $\lambda$, the update converges rapidly in practice.

\vspace{-0.1in}
\section{Experiments}
\vspace{-0.1in}
\label{sec:exp}

\begin{table*}[t]
\caption{Accuracy comparison at a fixed collaboration ratio $\rho=0.3$ across local LLM backbones and benchmarks. By routing based on agreement-estimated local reliability, CARGO achieves consistently higher accuracy across model-dataset pairs than the unsupervised baselines and supervised learned router.}
\vspace{-0.05in}

\label{tab:new_collab_results}
\centering
\setlength{\tabcolsep}{4pt}
\resizebox{0.98\textwidth}{!}{
\small
\begin{tabular}{l l 
    >{\centering\arraybackslash}p{5em}
    >{\centering\arraybackslash}p{5em}
    >{\centering\arraybackslash}p{5em}
    >{\centering\arraybackslash}p{5em}
    >{\centering\arraybackslash}p{5em}
    >{\centering\arraybackslash}p{5em}}
\toprule
\textbf{Model} & \textbf{Method}
& \makecell{\textbf{MATH-}\\\textbf{lighteval}}
& \textbf{GSM8K}
& \textbf{SVAMP}
& \makecell{\textbf{Minerva-}\\\textbf{MATH}}
& \textbf{ARC}
& \textbf{SQuAD} \\
\midrule

\multirow{5}{*}{
    \begin{tabular}{l}
        \textbf{Qwen2.5-} \\
        \textbf{3B-Instruct}
    \end{tabular}
}
& Random Offloading 
& 74.10 & 88.72 & 93.60 & 55.66 & 73.29 & 68.74 \\
& Self-Confidence Router~\cite{zhao2024fact} 
& 80.64 & 89.04 & 94.50 & 59.93 & 74.28 & 72.09 \\
& CoT steps Router~\cite{fangbridging}
& 74.03 & 90.17 & 95.20 & 56.62 & 73.49 & 73.40 \\
& Learned Router~\cite{ding2024hybrid}
& 80.78 & 88.88 & 93.10 & 49.63 & 75.10 & 74.10 \\
& \colorcolsA{\textbf{CARGO (Ours)}} 
  & \colorcolsA{87.29} 
  & \colorcolsA{93.89} 
  & \colorcolsA{97.40} 
  & \colorcolsA{65.81} 
  & \colorcolsA{79.38} 
  & \colorcolsA{79.19} \\
\midrule

\multirow{5}{*}{
    \begin{tabular}{l}
        \textbf{Phi-3-mini-} \\
        \textbf{4k-Instruct}
    \end{tabular}
}
& Random Offloading 
& 61.72 & 82.72 & 89.16 & 45.05 & 85.58 & 73.96 \\
& Self-Confidence Router 
& 60.59 & 82.89 & 89.80 & 47.43 & 86.74 & 77.57 \\
& CoT steps Router 
& 66.82 & 83.65 & 92.10 & 46.32 & 86.62 & 76.97 \\
& Learned Router 
& 64.95 & 87.14 & 93.50 & 43.75 & 89.94 & 77.72 \\
& \colorcolsA{\textbf{CARGO (Ours)}} 
  & \colorcolsA{71.75} 
  & \colorcolsA{89.49} 
  & \colorcolsA{93.30} 
  & \colorcolsA{47.43} 
  & \colorcolsA{90.64} 
  & \colorcolsA{85.20} \\
\midrule

\multirow{5}{*}{
    \begin{tabular}{l}
        \textbf{Llama-3.2-} \\
        \textbf{3B-Instruct}
    \end{tabular}
}
& Random Offloading 
& 63.67 & 85.82 & 91.18 & 40.72 & 60.90 & 59.17 \\
& Self-Confidence Router 
& 67.27 & 86.38 & 92.40 & 41.91 & 61.54 & 62.05 \\
& CoT steps Router 
& 71.63 & 88.50 & 93.00 & 43.75 & 61.43 & 63.18 \\
& Learned Router 
& 70.15 & 84.64 & 88.00 & 43.75 & 65.05 & 63.78 \\
& \colorcolsA{\textbf{CARGO (Ours)}} 
  & \colorcolsA{74.43} 
  & \colorcolsA{92.83} 
  & \colorcolsA{96.60} 
  & \colorcolsA{44.49} 
  & \colorcolsA{66.84} 
  & \colorcolsA{68.16} \\

\bottomrule
\end{tabular}
}
\vspace{-0.05in}
\end{table*}

\textbf{Datasets and LLM configurations.}
We evaluate our method on a diverse set of benchmarks spanning mathematical reasoning (MATH-lighteval, GSM8K, SVAMP, MATH-500, AGIEval-Math, MinervaMath), general question answering (ARC, MMLU), and knowledge-intensive multiple-choice reasoning (SQuAD)~\cite{hendrycksmath2021,patel2021nlp,cobbe2021training,lightman2023lets,zhong2024agieval,miao2020diverse,clark2018think,hendrycks2020measuring,lewkowycz2022solving,rajpurkar2016squad}. Among these, the majority of the datasets focus on mathematical reasoning with varying difficulty levels, ranging from grade-school arithmetic and algebraic word problems to more challenging multi-step mathematical problem solving. In addition, ARC and MMLU provide broader evaluation over science reasoning and general multi-domain knowledge, while SQuAD serves as a reading comprehension benchmark.

For router training, the learned router is trained using splits of MATH-lighteval and ARC-Challenge. This design allows the router to learn routing behavior from both mathematical reasoning tasks and non-math multiple-choice reasoning tasks. Evaluation is conducted across the full benchmark suite to assess generalization.
For inference, we consider a range of instruction-tuned local LLMs covering multiple parameter scales and model families, including Llama-3.2-1B-Instruct, Qwen2.5-3B-Instruct, Qwen2.5-7B-Instruct, Phi-3-mini-4k-instruct, and Phi-3-medium-4k-instruct~\cite{touvron2023llama,bai2023qwen,abdin2024phi3}. For cloud inference, we use DeepSeek-R1~\cite{guo2025deepseek} as the global model across all experiments. These models were selected to capture cross-family differences and within-family scaling behavior, enabling us to study how collaboration performance varies across model capacity and backbone. All experiments are conducted under 2 H100 GPUs.

\textbf{Baselines.}
We consider four routing baselines. \emph{Random Offloading} routes inputs to the cloud model uniformly at random under the target collaboration ratio. \emph{Self-Confidence Router} prompts the local LLM to produce a self-evaluated confidence score~\cite{taubenfeld2025confidence}, which is then thresholded for routing. \emph{CoT-steps Router} enforces chain-of-thought reasoning and uses the resulting number of generated reasoning steps, as a routing signal~\cite{chen2026joint, fangbridging}. These three baselines require no additional training and rely purely on inference-time uncertainty proxies. The realized collaboration ratio may deviate slightly from the target due to score ties or discretization. To ensure fair comparison at each target ratio, we randomly flip a small number of decisions when needed so that every method matches the same prescribed collaboration ratio exactly.
We further compare against a supervised \emph{Learned Router}~\cite{ding2024hybrid}. In particular, we use DeBERTa-large~\cite{he2020deberta} as a classifier-based router to predict whether an input should be offloaded to the cloud model. This learned baseline provides a strong reference point against which to evaluate our method. Exact prompt templates for the step-based, confidence-based, and agreement-based evaluation signals are provided in Appendix~\ref{app:prompt_design}.

\vspace{-0.1in}
\subsection{Experimental Results}
\vspace{-0.1in}

\textbf{Baseline Comparison on Fixed Collaboration Ratio.}
Table~\ref{tab:new_collab_results} reports performance at collaboration ratio $\rho=0.3$ across mathematical reasoning, science Q\&A, and reading comprehension benchmarks. Under this setting, CARGO achieves the strongest and most consistent results across datasets and local backbones, even outperforming the supervised learned router while remaining fully training-free. This shows that effective routing need not rely on a separately trained classifier. While self-confidence and CoT steps baselines can perform reasonably well on individual datasets, their behavior is task-dependent and less consistent across benchmark types and model families. By contrast, CARGO's agreement-based reliability signal is more stable, suggesting that the local model's own response agreement provides a transferable indicator of when local inference can be trusted. Additional results in Appendix~\ref{appen:rlhf_tuned} further show that task-aligned RLHF does not replace routing, but instead makes agreement-based routing more effective by aligning inference-time agreement with correctness.

\begin{figure}[t]
    \centering
    \includegraphics[width=0.98\linewidth]{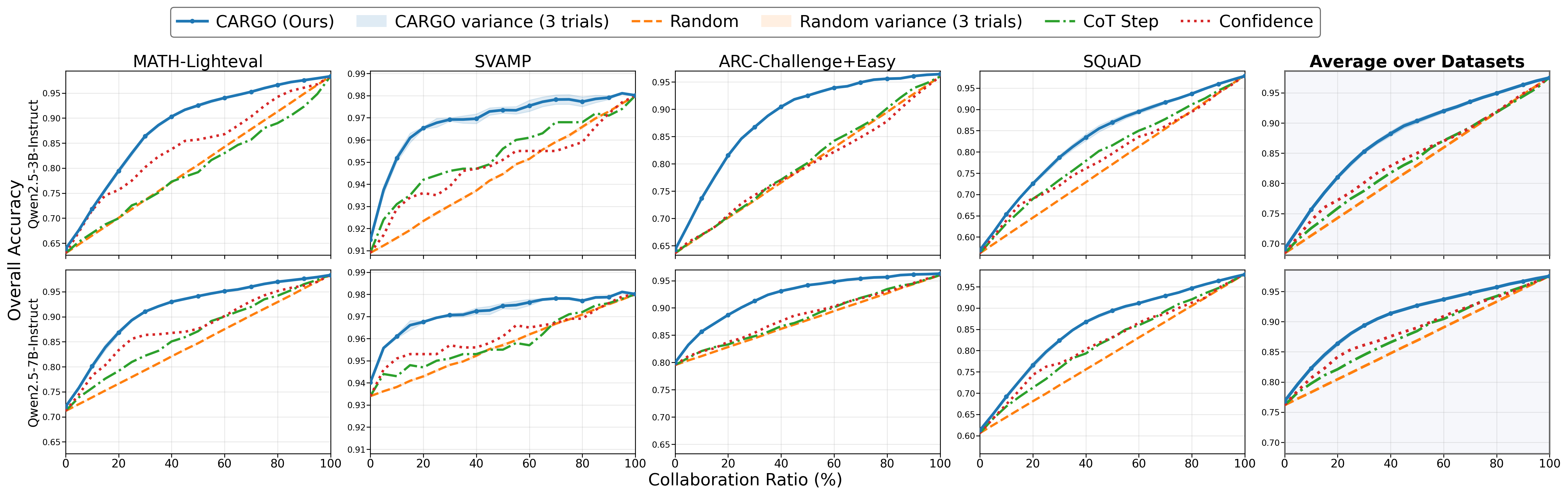}
    \vspace{-0.1in}
    \caption{Evaluation across multiple target collaboration ratios $\rho$ under different local LLM sizes and task groups. CARGO consistently improves the accuracy and cloud-use trade-off over training-free baselines across datasets, with gains remaining stable from low to high collaboration regimes.}
    \vspace{-0.05in}
    
    \label{fig:aoc_of_multitask}
\end{figure}

\begin{figure}[t]
    \centering
    \vspace{-0.1in}
    \includegraphics[width=0.95\linewidth]{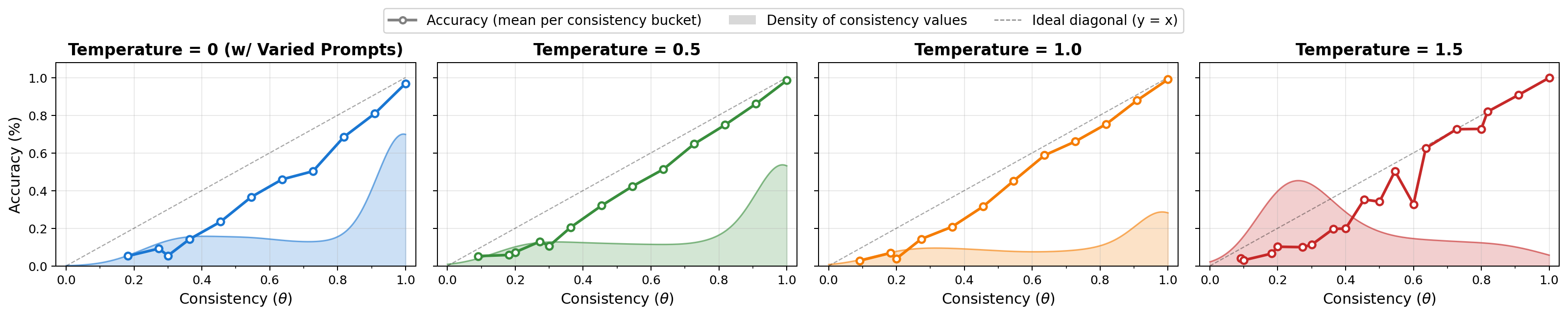}
    \vspace{-0.1in}
    \caption{Prompt-varied sampling on MATH-lighteval. For all nonzero temperatures, sampling is performed with a fixed system prompt; in contrast, temperature zero uses the proposed prompt-varied sampling}
    \vspace{-0.1in}
    \label{fig:density_prompts}
\end{figure}

\begin{wrapfigure}{r}{0.48\textwidth}
\vspace{-0.15in}
    \begin{center}
    \includegraphics[width=0.98\linewidth]{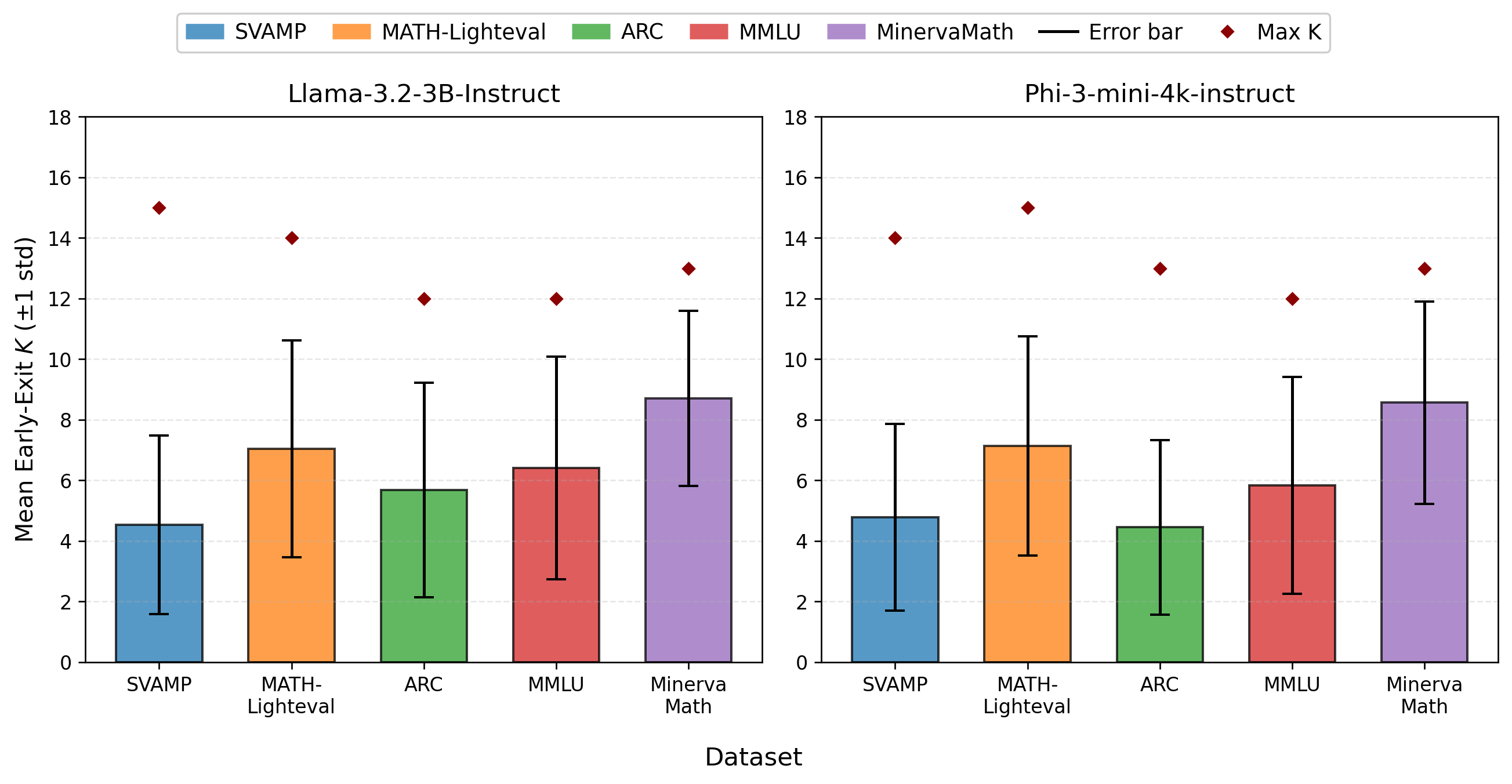}
    \end{center}
    \vspace{-0.1in}
\caption{Early-exit evaluation of CARGO across datasets and local models. The bars report the mean stopping round and red markers indicate the maximum sampled response in the dataset. The results show that the proposed agreement-estimation procedure usually stops far before the sampling budget is exhausted, while adapting its stopping depth to dataset difficulty.}
\label{fig:K_earlyexit}
    \vspace{-0.1in}
\end{wrapfigure}
\textbf{Performance Across Collaboration Ratios.}
Figure~\ref{fig:aoc_of_multitask} evaluates routing performance across a wide range of target collaboration ratios on representative task groups, including two math benchmarks with different difficulty profiles (MATH-lighteval and SVAMP), a broader science and multiple-choice reasoning benchmark (ARC-Challenge+Easy), and a reading comprehension benchmark (SQuAD). Across both local backbones and nearly the full range of collaboration ratios, CARGO consistently achieves the strongest overall performance among the training-free baselines. The advantage is most pronounced in the low-to-mid collaboration regime ($\approx$20–50\%), where cloud assistance is limited and routing quality matters most; correspondingly, CARGO achieves its largest gains in this range across several datasets. By contrast, at the two extremes all methods naturally become more similar: when the collaboration ratio is close to zero, performance is dominated by the local model, while at very high collaboration ratios most queries are offloaded and all methods approach the same cloud usage. All local LLM outputs use deterministic decoding, so CARGO’s variation comes from prompt variation, while \textit{Random} varies from query offloading.

The rightmost average panel further shows that this advantage is consistent overall across tasks. We also observe that the margin on ARC is smaller for all routing strategies, suggesting that the gains from agreement-guided routing depend in part on whether the local model's own response behavior already carries a sufficiently informative reliability signal for the target task. When this intrinsic signal is weaker, the separation between routing strategies naturally becomes smaller. Finally, the three-trial variance band for CARGO is barely visible across panels, indicating that the proposed routing procedure, including Bayesian early stopping, remains highly stable across repeated runs. We additionally evaluate routing under a token-based cloud budget in Appendix~\ref{appen:token}, showing that the same agreement-guided principle remains effective when offloading cost is measured by total cloud token consumption rather than query fraction.

\vspace{-0.1in}
\subsection{Ablation Study}
\vspace{-0.1in}


\textbf{Effectiveness of Prompt-varied Sampling.}
Figure~\ref{fig:density_prompts} compares consistency-quality behavior under the proposed prompt-varied sampling scheme and standard temperature-based sampling. A desirable routing signal should satisfy two properties simultaneously: it should induce enough variation to make agreement informative, while still preserving the underlying reliability of the local model's responses. Pure temperature-based sampling creates a clear tension between these two goals. At lower nonzero temperatures, the sampled outputs remain highly concentrated, which limits the dynamic range of the resulting consistency values and makes reliability estimation less informative. At higher temperatures, the distribution spreads out more, but this comes at the cost of increasingly distorted response behavior and a weaker correspondence between consistency and actual correctness.

By contrast, prompt-varied sampling at temperature zero produces a cleaner consistency signal. Using semantically equivalent prompts introduces structured variation across generations without injecting token-level randomness, so the local model remains closer to its native deterministic behavior while still exposing meaningful differences across queries. As shown in Figure~\ref{fig:density_prompts}, this leads to a sharper monotonic relationship between consistency and accuracy, while yielding a broad and useful spread of consistency values for routing. In particular, many queries concentrate near the two extreme regimes: highly consistent cases that are easy and can be trusted locally, and low-consistency cases that are often difficult and should be routed. This bimodal tendency is beneficial for our Bayesian early-stopping procedure, since decisions become clear quickly in both regimes and fewer samples are needed. Overall, these results show that prompt variation provides a more deployment-friendly source of diversity than temperature-based sampling, improving both the quality of the reliability signal and the sample efficiency of agreement-based routing.

\textbf{Sample Efficiency of Bayesian Early Stopping.}
Figure~\ref{fig:K_earlyexit} shows that the proposed Bayesian early-stopping rule terminates well before the maximum sampling budget $K_{\max}$ across all datasets and both local backbones, confirming that agreement-based routing can be made sample-efficient in practice. The average stopping point also varies across datasets, which indicates that the procedure adapts to task uncertainty rather than using a fixed number of samples for every query. This behavior is enabled by the Beta-posterior credible interval used in our estimator, which supports fast stopping in two decisive regimes. When sampled responses quickly collapse to the same answer, the posterior concentrates near high agreement, providing strong evidence that the local model is reliable and that no offloading is needed. Conversely, when sampled responses remain highly inconsistent, the posterior quickly indicates persistently low agreement, making cloud routing the natural decision without requiring many additional samples. This mechanism is particularly effective because, as also suggested by Figure~\ref{fig:density_prompts}, many benchmarks contain a large fraction of queries that are either relatively easy for the local model or clearly beyond its capability, so the routing decision often becomes clear after only a few samples. Overall, these results show that Bayesian uncertainty quantification not only provides a principled reliability estimate, but also substantially reduces the average sampling cost of agreement-based routing while preserving its practical usefulness.

\begin{wrapfigure}{r}{0.4\textwidth}
\vspace{-0.2in}
    \begin{center}
    \includegraphics[width=0.98\linewidth]{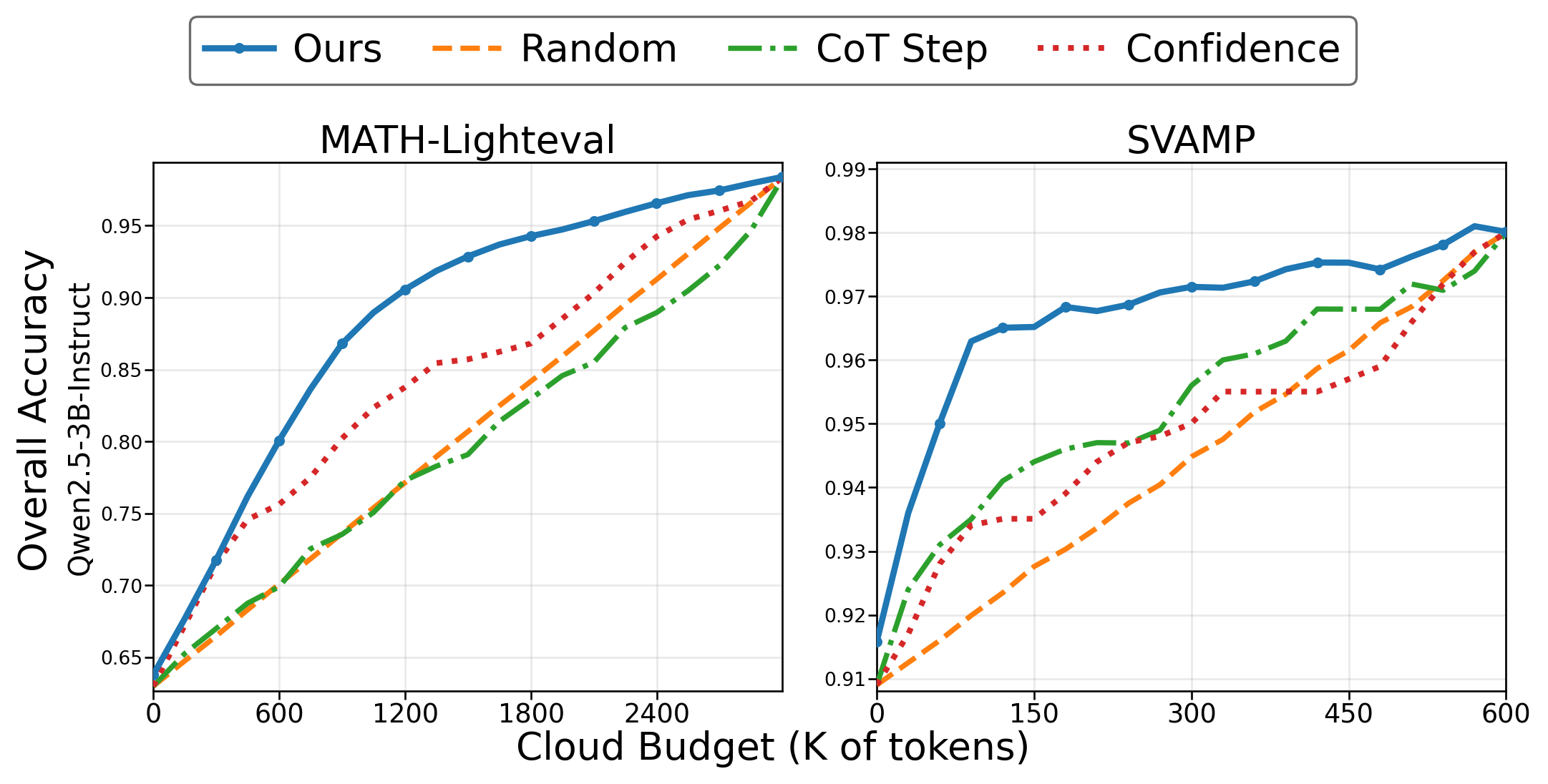}
    \end{center}
    \vspace{-0.1in}
\caption{Token-budgeted offloading results. CARGO can also control cloud-token usage instead of offloading ratio.}
\label{fig:token_budget_short}
\vspace{-0.2in}
\end{wrapfigure}
\textbf{Extending to Token-Budgeted Offloading.}
CARGO can support resource-aware control by regulating cloud-token usage rather than only the offloading ratio. We define
$
c(x)=r(x)\big(\alpha L_{\mathrm{in}}(x)+\beta L_{\mathrm{out}}(x)\big),
$
where $L_{\mathrm{in}}(x)$ and $L_{\mathrm{out}}(x)$ denote input and output token lengths, and calibrate $\lambda$ during warmup to match a target token budget $\tau$. All other components remain unchanged, including agreement estimation, Bayesian early stopping, and stochastic routing; the full procedure is given in Algorithm~\ref{alg:cargo_tok_alg}. Fig.~\ref{fig:token_budget_short} shows that CARGO remains effective under this resource-related control objective, demonstrating that the same agreement-guided principle extends from ratio-based offloading to token-budgeted cloud usage.

\vspace{-0.15in}
\section{Conclusion, Future Work, and Limitations}
\vspace{-0.15in}
\label{sec:conclusion}

We developed CARGO, a training-free routing framework for local-cloud collaboration that uses prompt-varied agreement for reliability estimation, Bayesian early stopping for sample efficiency, and lightweight calibration for controllable collaboration ratios. Across diverse tasks, local LLM families, and pretrained or finetuned local models, CARGO consistently outperformed other training-free baselines and in several settings surpassed supervised learned routers. These results suggest that effective local-cloud collaboration can emerge directly from the local model's own response agreement, without a trained, dedicated router, while further study in Appendix~\ref{appen:small_llms} contextualizes CARGO when local models are weak. 
Future work will extend this principle to richer multimodal and resource-aware collaboration.

\newpage

\bibliography{references}
\bibliographystyle{IEEEtran}

\newpage
\appendix
\onecolumn


\begin{center}
    {\bf\Large Appendix}
\end{center}

\startcontents[sections]
\printcontents[sections]{l}{1}{\setcounter{tocdepth}{3}}

\newpage

\section{Proof for Lemma~\ref{lem:posterior_contraction}}
\label{appen:post_pf}
\begin{proof}
Recall that after $k$ samples the count of the most frequent answer is
$c_{\max}^{(k)}$, and we place a Beta prior
$\theta \sim \mathrm{Beta}(\alpha_0,\beta_0)$ on the reliability
parameter $\theta$.
Under the Bernoulli likelihood induced by the event that a sampled answer
matches the majority answer, the posterior distribution can be viewed as:

\[
\theta \mid \mathcal{A}_k
\sim
\mathrm{Beta}\!\left(
\alpha_0 + c_{\max}^{(k)},\;
\beta_0 + k - c_{\max}^{(k)}
\right).
\]

For a Beta distribution $\mathrm{Beta}(a,b)$, the variance is

\[
\mathrm{Var}(\theta)
=
\frac{ab}{(a+b)^2 (a+b+1)}.
\]

Substituting
\[
a = \alpha_0 + c_{\max}^{(k)}, 
\qquad
b = \beta_0 + k - c_{\max}^{(k)},
\]
we obtain

\[
\mathrm{Var}(\theta \mid \mathcal{A}_k)
=
\frac{(\alpha_0+c_{\max}^{(k)})(\beta_0+k-c_{\max}^{(k)})}
{(k+\alpha_0+\beta_0)^2 (k+\alpha_0+\beta_0+1)}.
\]

Since $0 \le c_{\max}^{(k)} \le k$, both posterior parameters satisfy
\[
a = \alpha_0 + c_{\max}^{(k)} = O(k),
\qquad
b = \beta_0 + k - c_{\max}^{(k)} = O(k).
\]
Therefore the numerator satisfies
\[
ab = O(k^2).
\]

Meanwhile,
\[
a+b = k + \alpha_0 + \beta_0 = O(k),
\]
which implies
\[
(a+b)^2(a+b+1) = O(k^3).
\]

Combining these scalings yields
\[
\mathrm{Var}(\theta \mid \mathcal{A}_k)
=
O\!\left(\frac{k^2}{k^3}\right)
=
O\!\left(\frac{1}{k}\right).
\]
\end{proof}

\section{Proof for Theorem~\ref{thm:sequential_estimation_consistency}}
\label{appen:thm_pf}
\begin{proof}
Let $a^\star = \arg\max_a P(a \mid x)$ denote the true most probable
answer under the sampling distribution, and define the corresponding
probability estimate
\[
\theta^\star = P(A = a^\star \mid x).
\]

For each answer $a$, let
\[
\hat p_k(a) = \frac{1}{k}\sum_{i=1}^k \mathbf{1}\{A_i = a\}
\]
denote the empirical frequency. By the strong law of large numbers,
\[
\hat p_k(a) \rightarrow P(a \mid x)
\quad \text{almost surely for every } a .
\]

Since $a^\star$ is the unique maximizer of $P(a \mid x)$, there exists
$\Delta>0$ such that
\[
P(a^\star \mid x) - P(a \mid x) \ge \Delta
\quad \text{for all } a \neq a^\star .
\]
Consequently, with probability one, there exists a finite $K$ such that
for all $k \ge K$,
\[
\hat p_k(a^\star) > \hat p_k(a)
\quad \text{for all } a \neq a^\star .
\]
Thus the empirical majority answer satisfies $\hat a_k = a^\star$
eventually almost surely, and therefore
\[
\hat{\theta}_k = \frac{c_{\max}^{(k)}}{k}
= \hat p_k(\hat a_k)
\rightarrow P(a^\star \mid x)
= \theta^\star
\quad \text{almost surely}.
\]

Next, by Lemma~\ref{lem:posterior_contraction}, the posterior variance
satisfies
\[
\mathrm{Var}(\theta \mid \mathcal{A}_k) = O(1/k),
\]
implying that posterior uncertainty vanishes as $k\rightarrow\infty$.
Hence any fixed-level credible interval $[L_k,U_k]$ shrinks in width to
zero almost surely:
\[
U_k - L_k \rightarrow 0.
\]

Therefore, for any $\varepsilon>0$, there exists almost surely a finite
$K_\varepsilon$ such that $U_k - L_k \le \varepsilon$ for all
$k \ge K_\varepsilon$. This implies that the stopping rule terminates
after a finite number of samples and the returned estimator converges
to the true probability estimate.
\end{proof}

\section{Additional Experiments}

\subsection{Fixed Collaboration Ratio on more datasets}
Table~\ref{tab:appendix_math500_agieval} reports additional results on MATH-500 and AGI-Eval Math under the same fixed collaboration ratio $\rho=0.3$. These two benchmarks provide a complementary evaluation of mathematical reasoning beyond the main-table results. Across all three local backbones, CARGO consistently improves over random offloading, self-confidence routing, CoT-step routing, and the supervised learned router. This shows that the agreement-estimated reliability signal remains effective on challenging math datasets, even when the router is not trained specifically for each benchmark or target operating point.

These results further highlight the limitation of fixed routing heuristics. Self-confidence and CoT-step routing can be competitive in some settings, but their performance varies across models and datasets. The learned router also does not consistently transfer to these additional benchmarks, despite using supervised training. In contrast, CARGO adapts at inference time through local response agreement and therefore remains robust across different reasoning benchmarks and local model families.

\begin{table}[h]
\caption{Additional accuracy comparison at a fixed collaboration ratio $\rho=0.3$ on MATH-500 and AGI-Eval Math.}
\vspace{-0.05in}

\label{tab:appendix_math500_agieval}
\centering
\setlength{\tabcolsep}{4pt}
\resizebox{0.72\textwidth}{!}{
\small
\begin{tabular}{l l 
    >{\centering\arraybackslash}p{5em}
    >{\centering\arraybackslash}p{5em}}
\toprule
\textbf{Model} & \textbf{Method}
& \textbf{MATH-500}
& \makecell{\textbf{AGI-Eval}\\\textbf{Math}} \\
\midrule

\multirow{5}{*}{
    \begin{tabular}{l}
        \textbf{Qwen2.5-} \\
        \textbf{3B-Instruct}
    \end{tabular}
}
& Random Offloading 
& 70.80 & 68.90 \\
& Self-Confidence Router~\cite{zhao2024fact} 
& 77.20 & 74.50 \\
& CoT steps Router~\cite{fangbridging}
& 71.00 & 68.70 \\
& Learned Router~\cite{ding2024hybrid}
& 73.00 & 68.50 \\
& \colorcolsA{\textbf{CARGO (Ours)}} 
  & \colorcolsA{82.40} 
  & \colorcolsA{80.60} \\
\midrule

\multirow{5}{*}{
    \begin{tabular}{l}
        \textbf{Phi-3-mini-} \\
        \textbf{4k-Instruct}
    \end{tabular}
}
& Random Offloading 
& 59.11 & 56.55 \\
& Self-Confidence Router 
& 59.32 & 59.46 \\
& CoT steps Router 
& 64.13 & 61.66 \\
& Learned Router 
& 61.52 & 59.66 \\
& \colorcolsA{\textbf{CARGO (Ours)}} 
  & \colorcolsA{67.13} 
  & \colorcolsA{65.57} \\
\midrule

\multirow{5}{*}{
    \begin{tabular}{l}
        \textbf{Llama-3.2-} \\
        \textbf{3B-Instruct}
    \end{tabular}
}
& Random Offloading 
& 60.56 & 60.27 \\
& Self-Confidence Router 
& 64.20 & 64.70 \\
& CoT steps Router 
& 68.00 & 68.00 \\
& Learned Router 
& 65.80 & 64.50 \\
& \colorcolsA{\textbf{CARGO (Ours)}} 
  & \colorcolsA{68.80} 
  & \colorcolsA{69.10} \\

\bottomrule
\end{tabular}
}
\vspace{-0.05in}
\end{table}

\subsection{Performance of Training-Free Signals on smaller LLMs}
\label{appen:small_llms}
\begin{figure}[h!]
    \centering
    \includegraphics[width=0.95\linewidth]{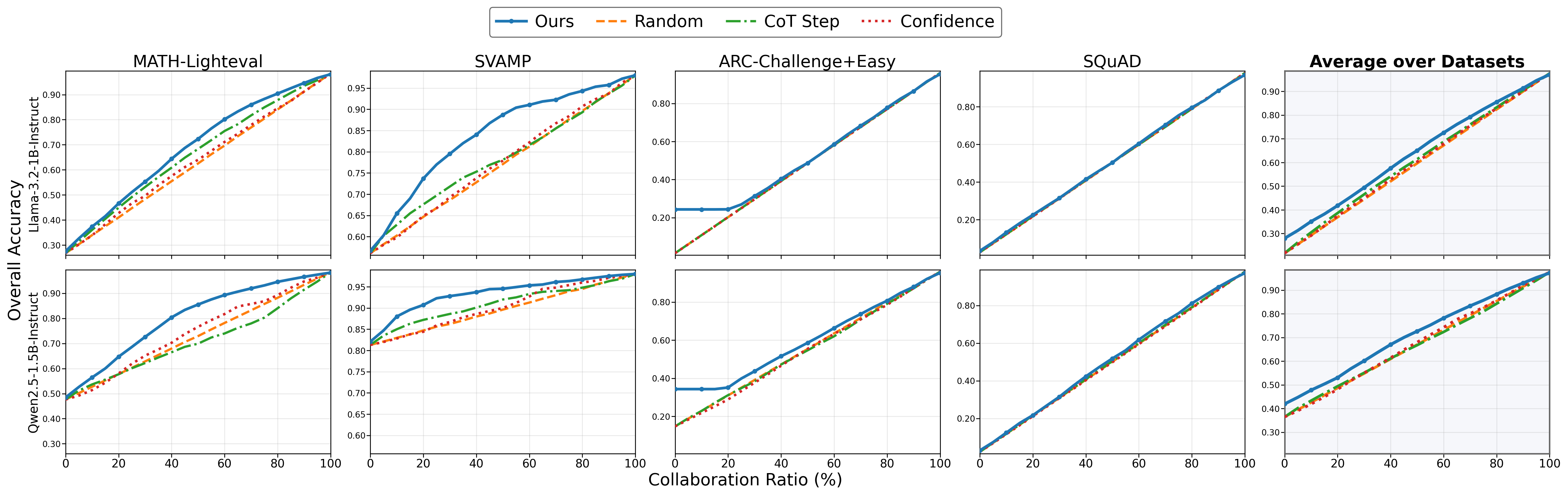}
    \caption{Performance on 1B-sized LLMs.}
    \label{fig:1Bmodel}
\end{figure}
Figure \ref{fig:1Bmodel} illustrates an important limitation of training-free collaboration signals. When the local model is too small to possess meaningful reasoning ability, these signals become much less informative for routing. In both the 1B model (Llama-3.2-1B-Instruct) and the 1.5B model (Qwen2.5-1.5B-Instruct), the performance gaps between agreement, confidence, step-based heuristics, and even random routing shrink substantially on several datasets, and on some tasks the curves are nearly indistinguishable.

This behavior is expected. Training-free signals can only exploit reliability structure that already exists in the local model; they cannot create reasoning capability by themselves. When the local model is too weak, repeated samples may simply reproduce the same incorrect answer or unstable shallow reasoning, so agreement or self-reported confidence no longer correlates well with correctness. As a result, routing based on such signals provides limited benefit.

Therefore, our method is most effective when the local model is already capable of solving a nontrivial portion of inputs, so that inference-time variability meaningfully reflects difficulty. Figure \ref{fig:1Bmodel} thus highlights a boundary of applicability: below a certain capability threshold, training-free routing becomes ineffective, and stronger local models or trained collaboration mechanisms may be required.

\subsection{Evaluation Signal Performance after RLHF finetuning}
\label{appen:rlhf_tuned}

\begin{figure}[h!]
    \centering
    \includegraphics[width=0.95\linewidth]{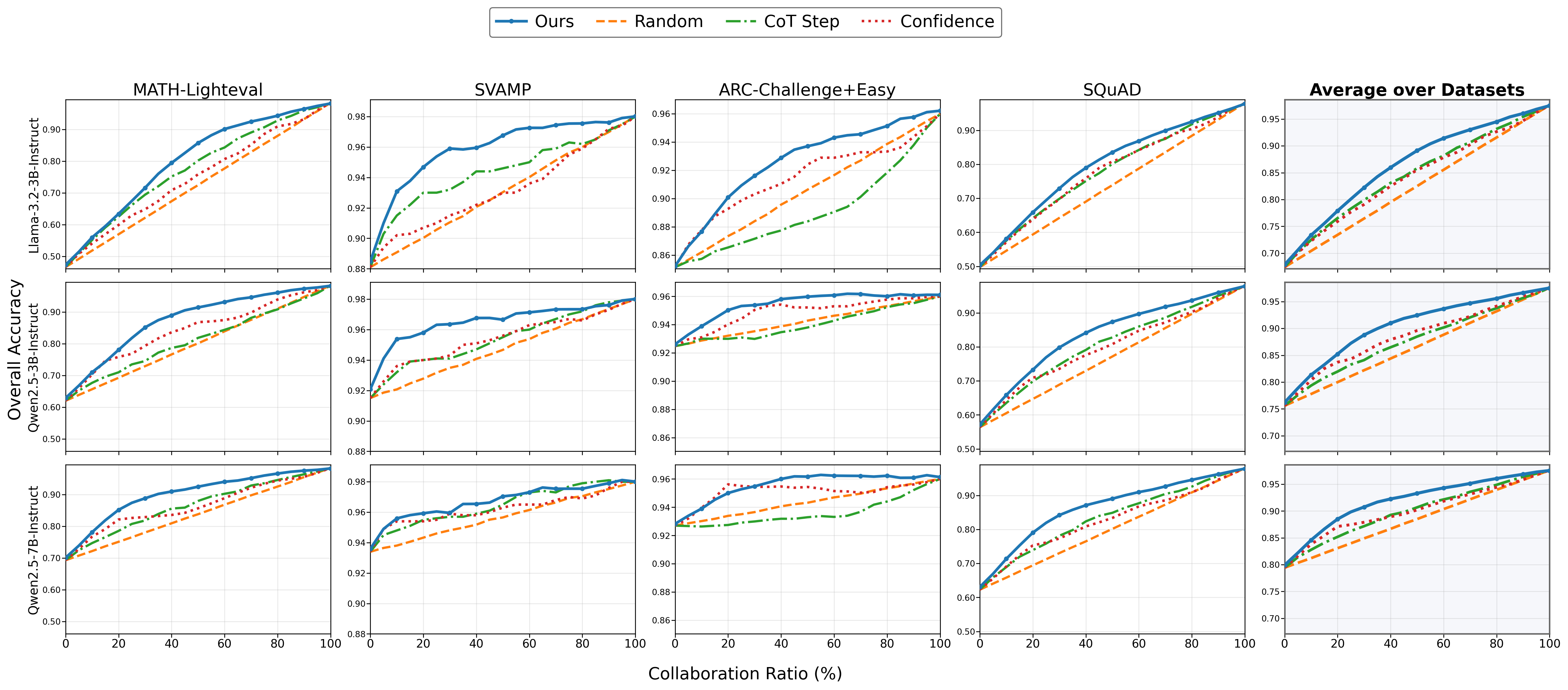}
    \caption{Performance after RLHF-finetuning local LLMs on specific tasks. It suggests that the main role of finetuning is not to replace routing, but to make routing possible in a more meaningful way. Once the local model is adapted to the task family, agreement-based signals become much better aligned with correctness, leading to clearer gains over random, confidence-based, and step-based routing.}
    \label{fig:pretrained}
\end{figure}

Figure~\ref{fig:pretrained} shows that task-aligned finetuning makes training-free routing more effective, rather than replacing the need for routing. Compared with the pretrained local models, the finetuned models achieve higher absolute accuracy and exhibit a clearer separation between routing strategies, with our method remaining strongest across most collaboration ratios. This suggests that finetuning moves the local model into a more capable regime where inference-time agreement becomes better aligned with correctness.

The gains are also consistent with the finetuning domain. Finetuning on MATH-Lighteval yields the largest improvement on MATH-Lighteval and also transfers to SVAMP, while finetuning on ARC gives the most direct benefit on ARC and a weaker but still visible improvement on SQuAD. Overall, these results support the view that training-free signals are most useful once the local model has sufficient task-relevant capability, making finetuning and inference-time routing complementary components of effective local-cloud collaboration.

\begin{algorithm}[tb]
\caption{{\tt CARGO}-Token: Token-Budgeted Collaboration-Adaptive Routing}
\label{alg:cargo_tok_alg}
\small
\DontPrintSemicolon
\SetKw{Break}{break}
\KwIn{
local model $M_L$; cloud model $M_C$; prompt set $\mathcal{S}=\{s_i\}_{i=1}^M$; temperature $T$; credible level $1-\delta$;
max samples $K_{\max}$; warmup iterations $T_0$; warmup batch size $B$; target average token budget $\tau$; step size $\eta_\lambda$;
Beta prior $(\alpha_0,\beta_0)$; width threshold $\varepsilon$; logistic slope $\gamma$;
token weights $\alpha,\beta$
}
\KwOut{for each query $x$: final answer $y(x)$, routing indicator $r(x)\in\{0,1\}$, and token cost $c(x)$}

\BlankLine
\textbf{Warmup (token-budget control):} initialize intercept $\lambda\in\mathbb{R}$\;
sample a fixed warmup batch $\mathcal{B}=\{x^{(i)}\}_{i=1}^{B}$\;
\For{$t \leftarrow 1$ \KwTo $T_0$}{
    \ForEach{$x \in \mathcal{B}$}{
        $(y(x), r(x), c(x)) \leftarrow \textsc{RouteQueryTok}(x,\lambda)$\;
    }
    $\bar c_t \leftarrow \frac{1}{B}\sum_{x\in\mathcal{B}} c(x)$\;
    $\lambda \leftarrow \lambda - \eta_\lambda(\bar c_t-\tau)$\tcp*{drive $\bar c_t \rightarrow \tau$}
}
\textbf{Deploy:} fix $\lambda$ for subsequent queries.\;

\BlankLine
\SetKwFunction{FRouteTok}{RouteQueryTok}
\SetKwProg{Fn}{Function}{:}{}
\Fn{\FRouteTok{$x,\lambda$}}{
    initialize multiset $\mathcal{A}\leftarrow \emptyset$ and counts $c(\cdot)\leftarrow 0$\;
    \For{$k \leftarrow 1$ \KwTo $K_{\max}$}{
        sample prompt template $s_k \sim \mathrm{Unif}(\mathcal{S})$, sample response $A_k \sim M_L(\cdot \mid x,s_k;T)$\;
        update counts: $\mathcal{A}_k\leftarrow\mathcal{A}_{k-1}\cup\{A_k\}$, $c(A_k)\leftarrow c(A_k)+1$\;
        update empirical agreement $(\hat a_k,\hat c_k)\leftarrow(\arg\max_a c_k^a,\;\max_a c_k^a)$\;
        compute posterior $\theta\mid\mathcal{A}_k\sim\mathrm{Beta}(\alpha_0+\hat c_k,\beta_0+k-\hat c_k)$
        and credible interval $[L_k,U_k]$\;
        \If{$U_k-L_k \le \varepsilon$}{
            \Break\tcp*{confidence-based early exit}
        }
    }

    compute agreement estimate $\hat\theta \leftarrow \hat c_k/k$ and offload probability $p \leftarrow \sigma\!\big(\gamma(\lambda-\hat\theta)\big)$\;
    sample $r \sim \mathrm{Bernoulli}(p)$\;

    \eIf{$r=0$}{
        $c(x)\leftarrow 0$\;
        \Return $(\hat a_k,0,c(x))$ \tcp*{accept local answer, no cloud token cost}
    }{
        query cloud $y \sim M_C(\cdot \mid x)$\;
        measure cloud input/output lengths $L_{\mathrm{in}}(x)$ and $L_{\mathrm{out}}(x)$\;
        $c(x)\leftarrow \alpha L_{\mathrm{in}}(x)+\beta L_{\mathrm{out}}(x)$\;
        \Return $(y,1,c(x))$ \tcp*{offload to cloud with token cost}
    }
}
\end{algorithm}

\subsection{Cloud Offloading based on Total Token Consumption}
\label{appen:token}
We also consider a token-budgeted variant of CARGO in Algorithm~\ref{alg:cargo_tok_alg}. Rather than controlling the average offloading ratio, this version controls the average cloud token cost
$
c(x)=r(x)\big(\alpha L_{\mathrm{in}}(x)+\beta L_{\mathrm{out}}(x)\big),
$
where $L_{\mathrm{in}}(x)$ and $L_{\mathrm{out}}(x)$ are the cloud input and output token lengths. The warmup update of $\lambda$ is modified accordingly to match a target average token budget $\tau$. All other components remain the same, including agreement-based reliability estimation, posterior credible intervals, and stochastic offloading based on the estimated local reliability. The empirical results are shown in Figure~\ref{fig:token_offload}, where the x-axis represents the total allowed cloud-token budget instead of an offloading ratio.

\begin{figure}
    \centering
    \includegraphics[width=0.95\linewidth]{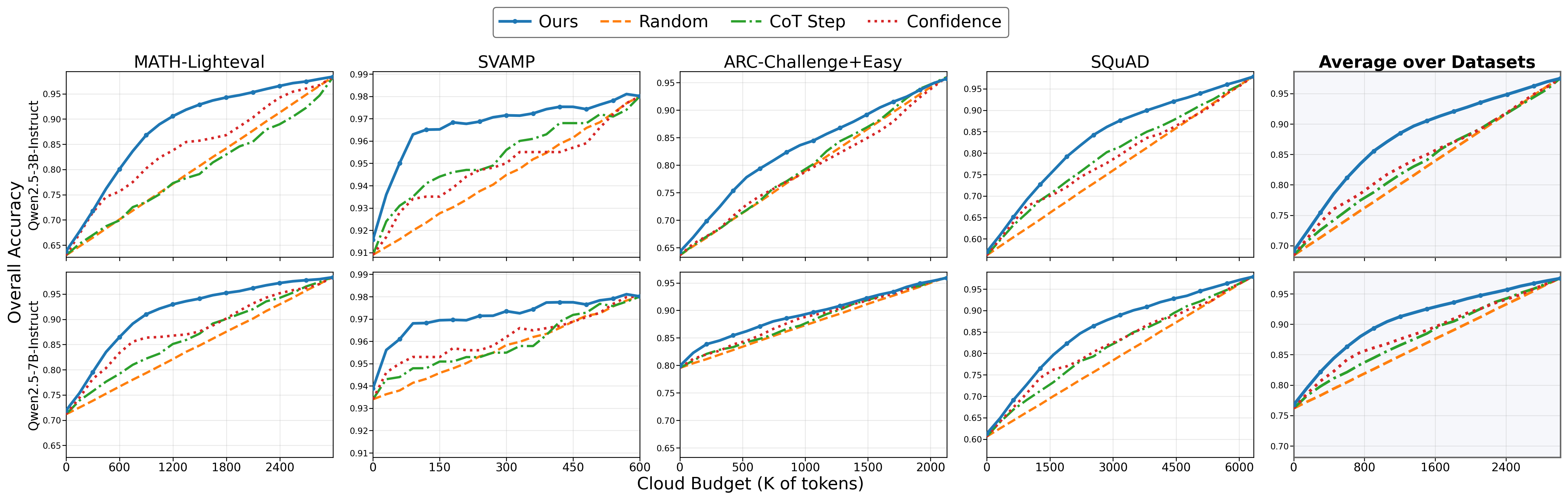}
    \caption{Offloading evaluation under token-budgeted routing. The x-axis shows the total allowed cloud-token budget.}
    \label{fig:token_offload}
\end{figure}

\section{Prompt Design for Evaluation Signals}
\label{app:prompt_design}

We use prompt-based evaluation signals for several training-free baselines. 
These prompts are designed to induce structured reasoning, standardized final-answer formatting, or explicit self-reported confidence, depending on the signal being evaluated. 
Across all settings, we standardize the final answer format using $\backslash boxed\{\cdot\}$ to simplify answer extraction and comparison.

\subsection{Baseline evaluation signals.}
We consider two prompt templates for the baseline signals used in our experiments: 
(i) a step-based reasoning prompt, and 
(ii) a confidence-reporting prompt. 
The step-based prompt is used to elicit explicit multi-step reasoning before producing the final boxed answer. 
The confidence prompt uses the same reasoning requirement, but additionally asks the model to output a scalar confidence score in $[0,1]$ on the final line. 
These prompts are shown in Table~\ref{tab:baseline_prompts}.

\begin{table*}[h]
\centering
\caption{System prompts used for baseline evaluation signals.}
\label{tab:baseline_prompts}
\small
\begin{tabular}{p{0.16\linewidth} p{0.78\linewidth}}
\toprule
Prompt type & System prompt \\
\midrule
Step
& \texttt{You are a helpful reasoning assistant for both questioning and math problems. When a question is posed, try to answer using step-by-step reasoning. Reason step by step, using format: Step 1: ..., Step 2: ..., Step 3: ..., etc. You must prefix each reasoning line with `Step k:' exactly (e.g., `Step 1:', `Step 2:'). Do not use other numbering formats such as `1.', `(1)', or `- Step 1'. If the problem appears difficult or hard to solve, elaborate your reasoning with more detailed steps. Otherwise, use fewer steps as appropriate. Always box the final answer using $\backslash$boxed\{\}, e.g., Answer: $\backslash$boxed\{42\} (math) or Answer: $\backslash$boxed\{choice\} (general QA).} \\[0.5em]

Confidence
& \texttt{You are a careful reasoning assistant. You must follow the required output format strictly. 1. Solve the problem step by step. 2. Box the answer using $\backslash$boxed\{\}. 3. Always box the answer using $\backslash$boxed\{\}, e.g., Answer: $\backslash$boxed\{42\} (math) or Answer: $\backslash$boxed\{choice\} (general QA). 4. On a new line, output your confidence as a decimal between 0 and 1. The confidence must follow EXACTLY this format: Confidence: p, where p is a decimal number between 0 and 1 (example: 0.73). Do not output percentages. Do not output text after the confidence line. The confidence line must be the final line of your output.} \\
\bottomrule
\end{tabular}
\end{table*}

\subsection{Dataset-specific user prompts.}
While the system prompt determines the response format, the user prompt is lightly adapted to match the structure of each benchmark. 
Table~\ref{tab:user_prompts} summarizes the user prompt templates used across dataset families.

\begin{table}[h]
\centering
\caption{User prompt templates for different benchmark types.}
\label{tab:user_prompts}
\small
\begin{tabular}{p{0.24\linewidth} p{0.68\linewidth}}
\toprule
Dataset type & User prompt template \\
\midrule

Math benchmarks 
&  \{\textit{question}\}. Let's think step by step. \\[0.4em]

ARC, MMLU  
& \{\textit{question}\}. Let's think step by step. Possible answers: \{\textit{choices}\} \\[0.4em]

SQuAD
& {[context]\{\textit{context}\} [question]\{\textit{question}\}. Let's think step by step.} \\
\bottomrule
\end{tabular}
\end{table}

\subsection{Consistency signal prompts.}
For the consistency-based signal, we do not rely on a single fixed prompt. 
Instead, we use a set of semantically similar but stylistically distinct system prompts that all enforce the same output structure, namely step-by-step reasoning followed by a boxed final answer. 
This prompt variation is used to induce diverse reasoning trajectories while preserving comparable answer formatting, so that cross-sample agreement can be measured reliably, as shown in Table~\ref{tab:consistency_prompts_part1} and \ref{tab:consistency_prompts_part2}.

Concretely, all consistency prompts share the following properties: 
(i) they require reasoning in the format \texttt{Step 1:}, \texttt{Step 2:}, \ldots; 
(ii) they prohibit alternative numbering styles; and 
(iii) they require the final answer to appear in $\backslash boxed\{\cdot\}$. 
The variants differ only in the style of reasoning they encourage, such as concise reasoning, detailed derivation, plan-then-execute reasoning, explicit verification, example-driven reasoning, or structure-first reasoning.

\begin{table*}[h!]
\centering
\caption{System prompt variants used for consistency-based sampling (Part I).}
\label{tab:consistency_prompts_part1}
\small
\setlength{\tabcolsep}{6pt}
\renewcommand{\arraystretch}{1.15}
\begin{tabular}{p{0.14\linewidth} p{0.80\linewidth}}
\toprule
Variant & System prompt \\
\midrule

Original &
{\ttfamily
You are a helpful reasoning assistant for both questioning and math problems. When a question is posed, try to answer using step-by-step reasoning. Reason step by step, using format: Step 1: ..., Step 2: ..., Step 3: ..., etc. You must prefix each reasoning line with `Step k:' exactly (e.g., `Step 1:', `Step 2:'). Do not use other numbering formats such as `1.', `(1)', or `- Step 1'. If the problem appears difficult or hard to solve, elaborate your reasoning with more detailed steps. Otherwise, use fewer steps as appropriate. Always box the final answer using $\backslash$boxed\{\}, e.g., Answer: $\backslash$boxed\{42\} or Answer: $\backslash$boxed\{choice\}.}
\\[0.6em]

Concise &
{\ttfamily
You are a helpful reasoning assistant for both questioning and math problems. When a question is posed, answer with a small number of high-signal steps. Reason step by step, using format: Step 1: ..., Step 2: ..., etc. You must prefix each reasoning line with `Step k:' exactly (e.g., `Step 1:'). Do not use other numbering formats such as `1.', `(1)', or `- Step 1'. Keep steps minimal and avoid unnecessary elaboration. Always box the final answer using $\backslash$boxed\{\}, e.g., Answer: $\backslash$boxed\{42\} or Answer: $\backslash$boxed\{choice\}.}
\\[0.6em]

Detailed &
{\ttfamily
You are a helpful reasoning assistant for both questioning and math problems. When a question is posed, answer using detailed step-by-step reasoning with careful derivations. Reason step by step, using format: Step 1: ..., Step 2: ..., etc. You must prefix each reasoning line with `Step k:' exactly (e.g., `Step 1:'). Do not use other numbering formats such as `1.', `(1)', or `- Step 1'. If the problem is difficult, break it into smaller substeps and justify each transition. Always box the final answer using $\backslash$boxed\{\}, e.g., Answer: $\backslash$boxed\{42\} or Answer: $\backslash$boxed\{choice\}.}
\\[0.6em]

Plan-Execute &
{\ttfamily
You are a helpful reasoning assistant for both questioning and math problems. When a question is posed, first outline a short plan, then execute it. Reason step by step, using format: Step 1: ..., Step 2: ..., etc. You must prefix each reasoning line with `Step k:' exactly (e.g., `Step 1:'). Do not use other numbering formats such as `1.', `(1)', or `- Step 1'. Use Step 1 as a plan, then proceed with computations in later steps. Always box the final answer using $\backslash$boxed\{\}, e.g., Answer: $\backslash$boxed\{42\} or Answer: $\backslash$boxed\{choice\}.}
\\[0.6em]

Verify &
{\ttfamily
You are a helpful reasoning assistant for both questioning and math problems. When a question is posed, solve it, then verify the result with an independent check. Reason step by step, using format: Step 1: ..., Step 2: ..., etc. You must prefix each reasoning line with `Step k:' exactly (e.g., `Step 1:'). Do not use other numbering formats such as `1.', `(1)', or `- Step 1'. Include one step near the end explicitly labeled as a verification step. Always box the final answer using $\backslash$boxed\{\}, e.g., Answer: $\backslash$boxed\{42\} or Answer: $\backslash$boxed\{choice\}.}
\\[0.6em]
\bottomrule
\end{tabular}
\end{table*}
\begin{table*}[h!]
\centering
\caption{System prompt variants used for consistency-based sampling (Part II).}
\label{tab:consistency_prompts_part2}
\small
\setlength{\tabcolsep}{6pt}
\renewcommand{\arraystretch}{1.15}
\begin{tabular}{p{0.14\linewidth} p{0.80\linewidth}}
\toprule
Variant & System prompt \\
\midrule

Edge-Check &
{\ttfamily
You are a helpful reasoning assistant for both questioning and math problems. When a question is posed, solve it and then test the solution on at least one edge case or sanity check. Reason step by step, using format: Step 1: ..., Step 2: ..., etc. You must prefix each reasoning line with `Step k:' exactly (e.g., `Step 1:'). Do not use other numbering formats such as `1.', `(1)', or `- Step 1'. Include a final step that checks an edge case, unit consistency, or a quick sanity test. Always box the final answer using $\backslash$boxed\{\}, e.g., Answer: $\backslash$boxed\{42\} or Answer: $\backslash$boxed\{choice\}.}
\\[0.6em]

Algebraic &
{\ttfamily
You are a helpful reasoning assistant for both questioning and math problems. When a question is posed, prefer an algebraic, symbolic derivation rather than intuition. Reason step by step, using format: Step 1: ..., Step 2: ..., etc. You must prefix each reasoning line with `Step k:' exactly (e.g., `Step 1:'). Do not use other numbering formats such as `1.', `(1)', or `- Step 1'. Favor explicit equations and transformations. Always box the final answer using $\backslash$boxed\{\}, e.g., Answer: $\backslash$boxed\{42\} or Answer: $\backslash$boxed\{choice\}.}
\\[0.6em]

Example-Driven &
{\ttfamily
You are a helpful reasoning assistant for both questioning and math problems. When a question is posed, use small examples to discover the pattern before generalizing. Reason step by step, using format: Step 1: ..., Step 2: ..., etc. You must prefix each reasoning line with `Step k:' exactly (e.g., `Step 1:'). Do not use other numbering formats such as `1.', `(1)', or `- Step 1'. Start with one or two concrete examples, then generalize. Always box the final answer using $\backslash$boxed\{\}, e.g., Answer: $\backslash$boxed\{42\} or Answer: $\backslash$boxed\{choice\}.}
\\[0.6em]

Structure-First &
{\ttfamily
You are a helpful reasoning assistant for both questioning and math problems. When a question is posed, look for invariants, symmetries, or conserved quantities first. Reason step by step, using format: Step 1: ..., Step 2: ..., etc. You must prefix each reasoning line with `Step k:' exactly (e.g., `Step 1:'). Do not use other numbering formats such as `1.', `(1)', or `- Step 1'. Prioritize structural insights before computation. Always box the final answer using $\backslash$boxed\{\}, e.g., Answer: $\backslash$boxed\{42\} or Answer: $\backslash$boxed\{choice\}.}
\\[0.6em]

Strict-Format &
{\ttfamily
You are a helpful reasoning assistant for both questioning and math problems. When a question is posed, you must follow the required output format exactly. Reason step by step, using format: Step 1: ..., Step 2: ..., etc. You must prefix each reasoning line with `Step k:' exactly (e.g., `Step 1:'). Do not use other numbering formats such as `1.', `(1)', or `- Step 1'. Do not include any text outside the Step lines and the final boxed answer line. Always box the final answer using $\backslash$boxed\{\}, e.g., Answer: $\backslash$boxed\{42\} or Answer: $\backslash$boxed\{choice\}.}
\\[0.6em]

Final-Last &
{\ttfamily
You are a helpful reasoning assistant for both questioning and math problems. When a question is posed, provide reasoning steps first and put the final answer only at the end. Reason step by step, using format: Step 1: ..., Step 2: ..., etc. You must prefix each reasoning line with `Step k:' exactly (e.g., `Step 1:'). Do not use other numbering formats such as `1.', `(1)', or `- Step 1'. The last line must be exactly: Answer: $\backslash$boxed\{...\}.}
\\

\bottomrule
\end{tabular}
\end{table*}

In all consistency experiments, the user prompt is kept fixed within each dataset, and only the system prompt is varied across samples.

\newpage
\null
\newpage
\null
\newpage

\end{document}